\documentclass[10pt,twocolumn,letterpaper]{article}

\usepackage[pagenumbers]{cvpr}
\usepackage[mathcal]{eucal}
\usepackage{bm}
\usepackage{multirow}
\usepackage{makecell}
\usepackage{colortbl}

\definecolor{cvprblue}{rgb}{0.21,0.49,0.74}
\usepackage[breaklinks,colorlinks,allcolors=cvprblue]{hyperref}
\hypersetup{
    pdftitle={TANGO: Training-free Controllable Human Motion Generation with Progressive Token Guidance},
    pdfauthor={Nhat Le, Daochang Liu, Baoru Huang, Anh Nguyen, Ajmal Mian}
}

\newcommand{\methodName}{TANGO\xspace}
\definecolor{TableSectionGray}{gray}{0.93}
\newcommand{\sectionrow}[1]{\rowcolor{TableSectionGray}\multicolumn{12}{c}{\bfseries #1}\\}

\title{\textit{\methodName}: Training-free Controllable Human Motion Generation with\\ Progressive Token Guidance}
\author{Nhat Le$^{1}$ \quad
Daochang Liu$^{1}$ \quad
Baoru Huang$^{2}$ \quad
Anh Nguyen$^{2}$ \quad
Ajmal Mian$^{1}$\\
$^{1}$The University of Western Australia, Crawley WA 6009, Australia\\
$^{2}$The University of Liverpool, Liverpool L69 3DR, United Kingdom}

\begin{document}
\maketitle

\begin{abstract}

\vspace{-4mm}

We present \methodName, a progressive token guidance framework for training-free controllable human motion generation. Unlike control approaches that primarily operate at a fixed temporal resolution, \methodName~leverages a multi-scale hierarchy of discrete motion tokens to establish a structured control process, progressively guiding global motion structure before refining local details. At each scale, the differentiable control objective guides all current-scale token distributions, and the sampled tokens condition subsequent finer-scale predictions. To address the challenge of discrete categorical sampling, we adopt a first-order approach following the Bayesian principle, adjusting token sampling distributions with a single first-order update per scale. To further improve guidance consistency, we introduce Within-Scale Interdependency (WSI) guidance to account for how token choices jointly affect the control objective, and Cross-Scale Consistency (CSC) guidance to estimate how remaining finer scales may affect the final control outcome. Together, these complementary guidance mechanisms enable consistent and effective progressive control without additional control-specific training. 
Experiments demonstrate the effectiveness of \methodName across different control objectives in terms of motion quality, control accuracy, and generalization to unseen constraints.

\end{abstract}

\vspace{-5mm}
\section{Introduction}
\vspace{-2mm}
\label{sec:intro}

We study the problem of \emph{inference-time, training-free control} of human motion, where new differentiable control objectives are introduced at inference without additional control-specific training.
Recent works have primarily focused on generating human motion from text~\cite{tevet2023mdm,zhang2024motiondiffuse,zhang2023t2m_gpt,jiang2023motiongpt,zhang2024motiongpt,guo2024momask,yuan2024mogents,lu2025scamo}.
Although natural language provides an intuitive interface for specifying high-level motion intent, it cannot precisely describe motion paths, joint locations, or environmental constraints~\cite{karunratanakul2023gmd,karunratanakul2024dno,xie2024omnicontrol}.
Moreover, exhaustively covering the broad space of possible control objectives during training is impractical.
This motivates flexible motion control that adapts a pretrained generative model to user-defined constraints at inference.

\begin{figure*}[t]
    \centering
    \includegraphics[width=\linewidth]{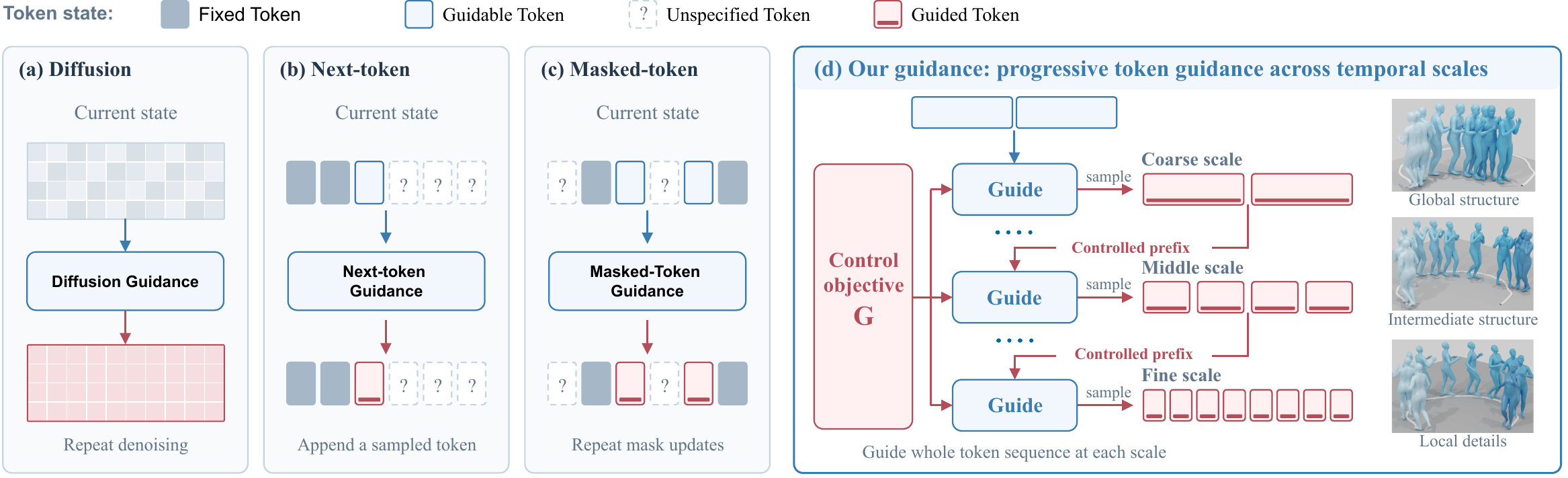}
    \vspace{-6mm}
\caption{
Comparison of guidance paradigms for motion control.
(a) Diffusion models~\cite{xie2024omnicontrol} repeatedly guide dense motion states during denoising, without explicitly separating global control semantics from local refinement.
(b) Next-token models~\cite{zhang2023t2m_gpt} guide tokens sequentially while future tokens remain unknown, so early errors can accumulate and propagate through the sequence.
(c) Masked-token models~\cite{pinyoanuntapong2025maskcontrol} iteratively guide subsets of unmasking tokens at a fixed temporal resolution.
(d) In contrast, \methodName guides all current-scale token distributions before sampling, with the controlled tokens conditioning subsequent finer-scale predictions.
This establishes a structured control process that progressively guides motion from global structure to fine-grained details.
}
    \label{fig:concept_comparisons}
    \vspace{-5mm}
\end{figure*}

Existing approaches~\cite{karunratanakul2023gmd,xie2024omnicontrol,karunratanakul2024dno,dai2024motionlcm,wan2024tlcontrol,pinyoanuntapong2025maskcontrol} face a trade-off between control flexibility, efficiency, and training requirements, as illustrated in Fig.~\ref{fig:concept_comparisons}.
Diffusion models support flexible spatial guidance through iterative denoising or noise optimization~\cite{tevet2023mdm,zhou2024emdm,karunratanakul2023gmd,karunratanakul2024dno,xie2024omnicontrol}, but can incur substantial inference cost.
Alternatively, token-based models generate motion through autoregressive next-token prediction~\cite{zhang2023t2m_gpt,jiang2023motiongpt,zhang2024motiongpt,lu2025scamo} or iterative masked prediction~\cite{guo2024momask,pinyoanuntapong2024mmm,yuan2024mogents,guo2025snapmogen}.
Next-token models generate tokens one at a time without knowing future tokens, while masked-token control methods~\cite{wan2024tlcontrol,pinyoanuntapong2025maskcontrol} iteratively refine tokens at a fixed temporal resolution.
Moreover, some methods require additional control-specific training~\cite{dai2024motionlcm,xie2024omnicontrol,pinyoanuntapong2025maskcontrol}, limiting their flexibility for unseen objectives.
Many existing control methods operate within a fixed temporal resolution, without explicitly separating global motion structure from local motion details. Consequently, local refinements may interfere with global constraints, leading to inconsistent control guidance throughout generation.
This motivates a structured control process that progressively guides motion from global structure to fine-grained details while maintaining overall control consistency.

Coarse-to-fine motion modeling provides a promising foundation for structured control.
GenMM~\cite{li2023genmm} and MAMM~\cite{agata2025mamm} explore hierarchical motion synthesis and control through example-based matching and alignment, while MoScale~\cite{zheng2026moscale} and ScaleMoGen~\cite{hwang2026scalemogen} introduce multi-scale discrete representations for motion generation and editing.
However, these approaches do not explicitly formulate progressive guidance of discrete token distributions toward arbitrary test-time differentiable control objectives.
The key question is how to leverage the multi-scale motion hierarchy for effective control throughout the generation process.

In this paper, we present \methodName, a new progressive token guidance framework for training-free human motion control.
Building on multi-scale modeling~\cite{karras2018progressive_gan,tian2024var,zheng2026moscale,hwang2026scalemogen}, \methodName represents motion as a hierarchy of discrete token sequences at progressively finer temporal resolutions.
We observe that coarse scales primarily capture low-frequency global motion structure, whereas finer scales add higher-frequency local dynamics, as shown in Fig.~\ref{fig:spectrogram}.
Our key idea is to turn this representation into a control hierarchy.
At each scale, the control objective guides the current token distributions, and the guided tokens become the conditioning prefix for subsequent finer-scale predictions, progressively refining the motion toward the objective.

However, controlling discrete motion tokens is challenging due to non-differentiable categorical sampling.
To make discrete-token guidance tractable, we adopt a first-order approach following the Bayesian principle, which adjusts the sampling probabilities of candidate codebook entries at all token positions according to the control objective using a single update per scale.
Nevertheless, independent first-order guidance cannot explicitly capture the combined effects of different token choices under the control objective.
We therefore introduce \emph{Within-Scale Interdependency Guidance} (WSI), which accounts for interactions among current-scale tokens to improve coordinated control decisions.
Furthermore, guidance evaluated on intermediate motion reconstructions ignores the contributions of remaining finer scales, potentially causing inconsistency with the final control outcome.
To address this, we introduce \emph{Cross-Scale Consistency Guidance} (CSC), which estimates the expected contributions of remaining scales to improve guidance at intermediate stages.
Together, these mechanisms enable progressive control that accounts for dependencies within and across scales without control-specific training.
As illustrated in Fig.~\ref{fig:applications}, \methodName supports joint control, obstacle avoidance, and human-scene interaction using the same motion model with different test-time objectives.

\newlength{\imgsep}
\newlength{\colw}
\begin{figure*}[t]
    \centering
    \setlength{\imgsep}{0.9pt}
    \setlength{\colw}{\dimexpr(\linewidth - 2\imgsep)/3\relax}
    \includegraphics[width=\colw]{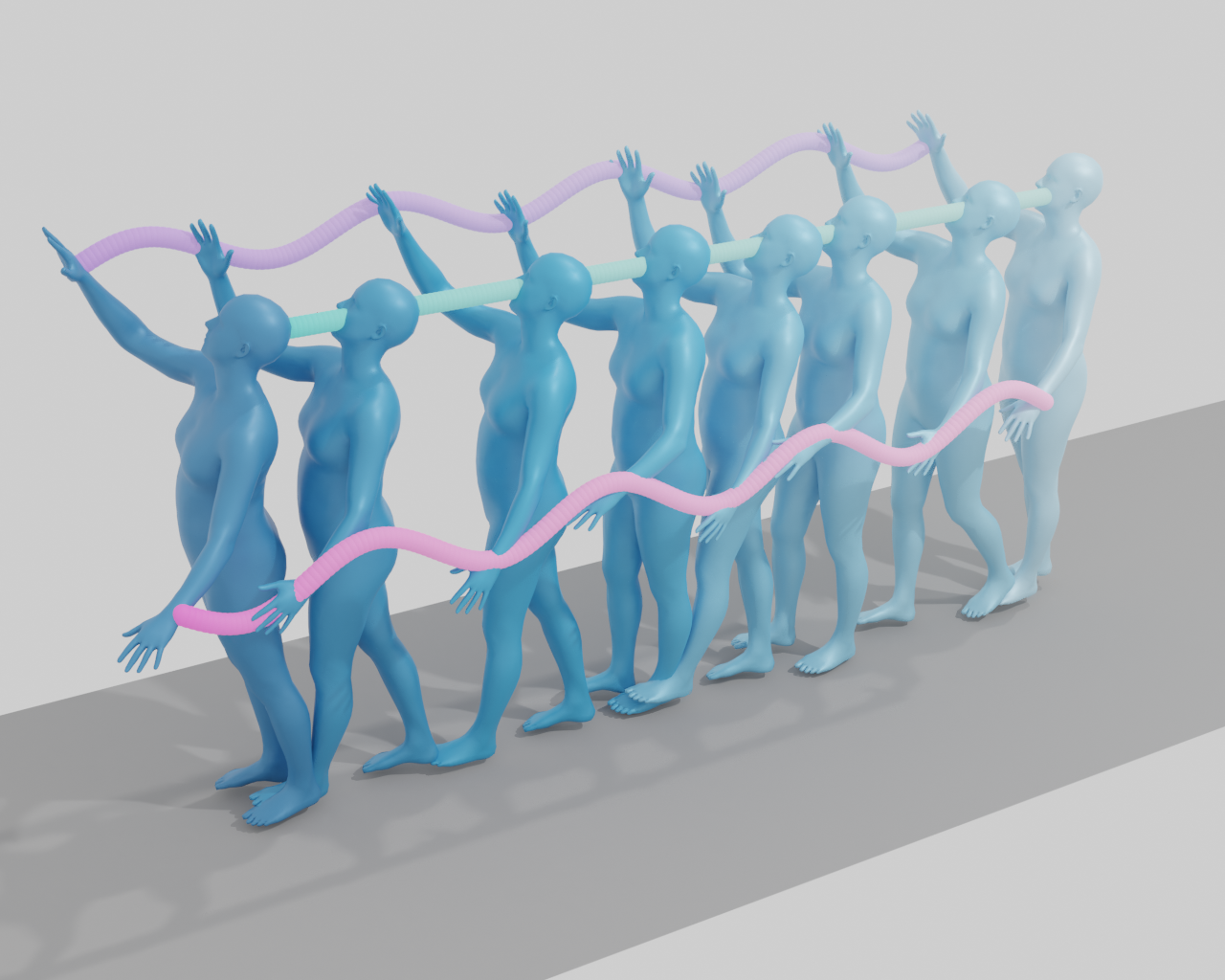}\hspace{\imgsep}%
    \includegraphics[width=\colw]{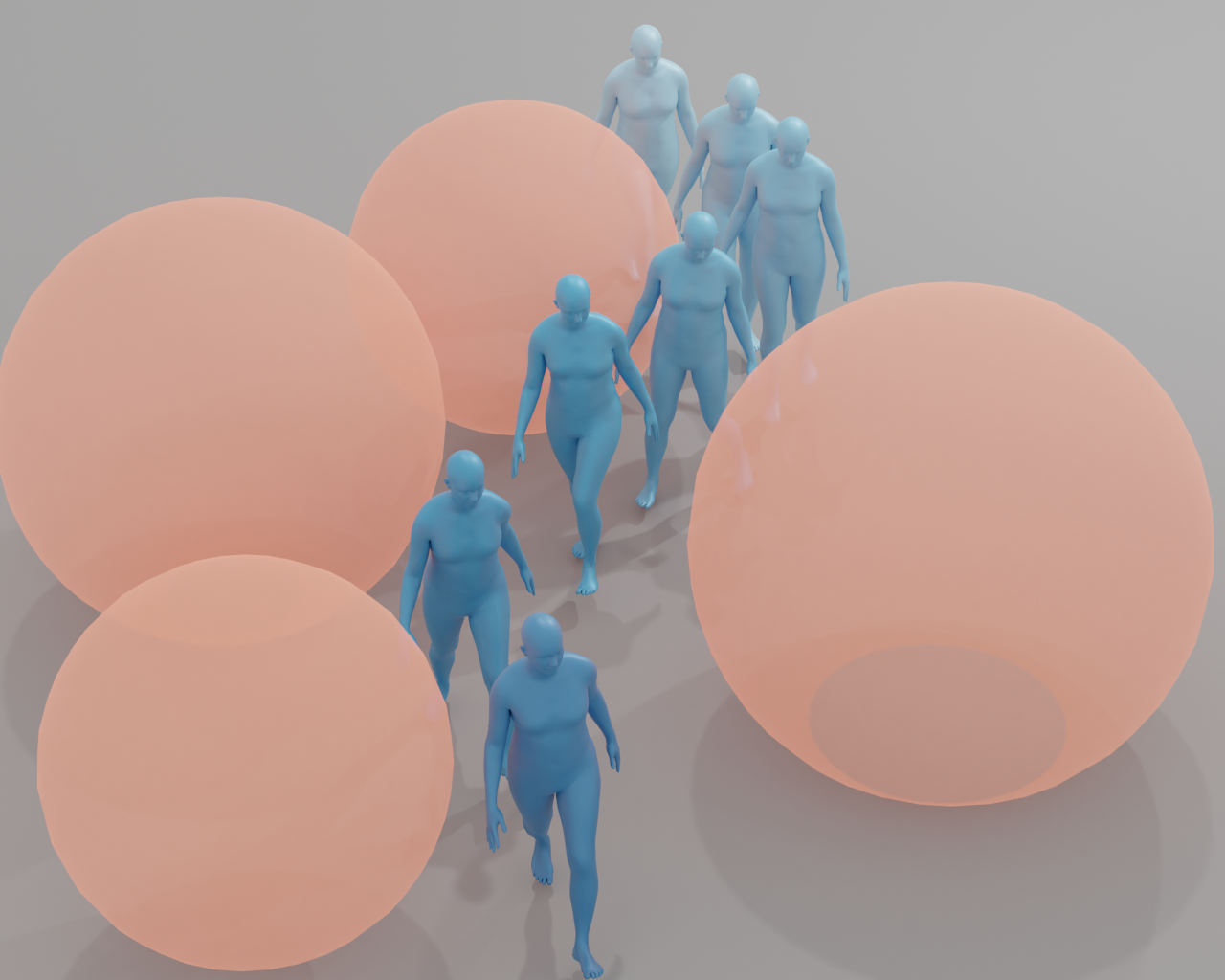}\hspace{\imgsep}%
    \includegraphics[width=\colw]{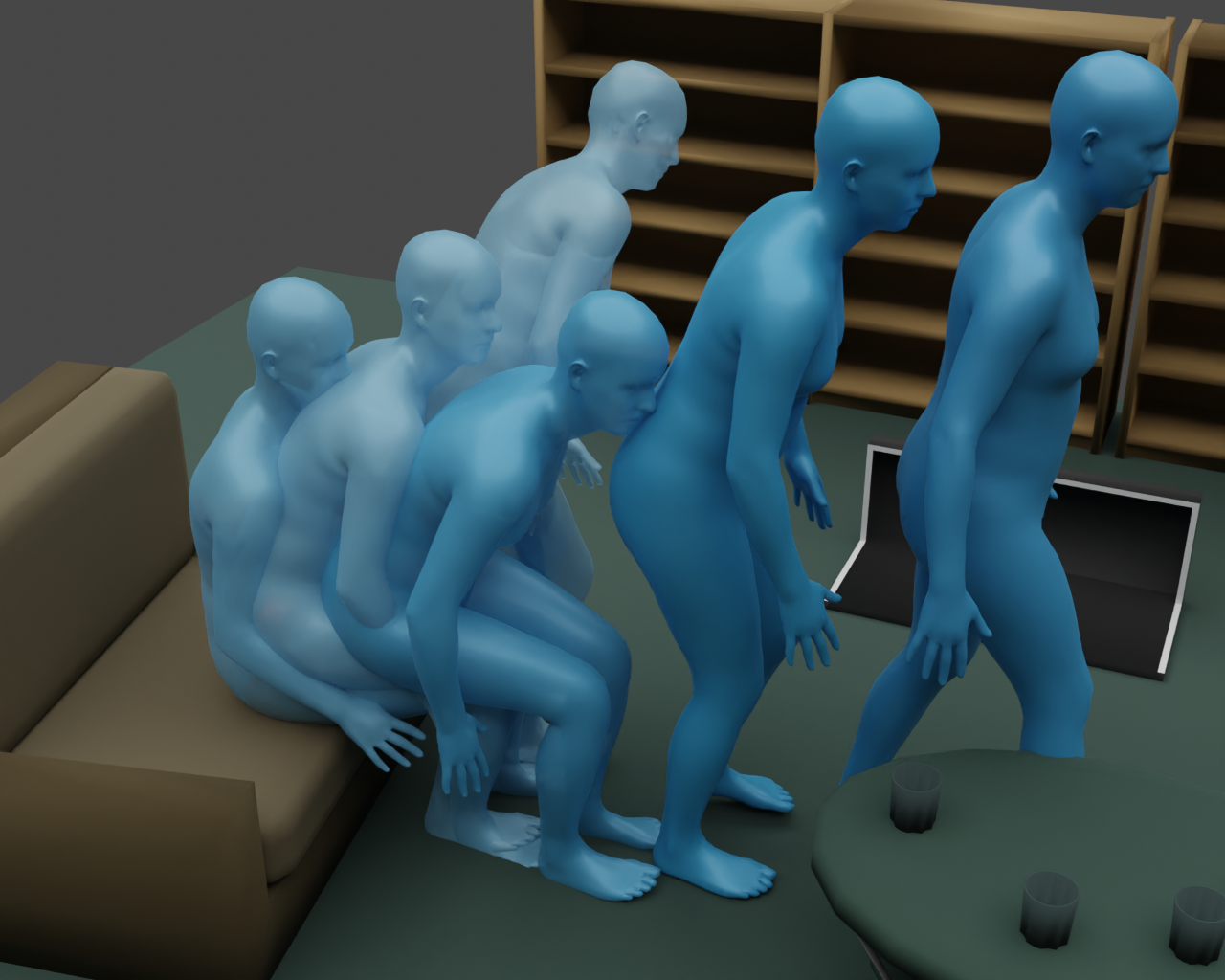}
    \makebox[\colw][c]{\footnotesize (a)}\hspace{\imgsep}%
    \makebox[\colw][c]{\footnotesize (b)}\hspace{\imgsep}%
    \makebox[\colw][c]{\footnotesize (c)}
    \vspace{-6mm}
    \caption{
    Example applications of \methodName for training-free controllable human motion generation:
    (a) Joint control, (b) Obstacle avoidance, and
    (c) Human-scene interaction.
    Best viewed in supplementary video.
    }
    \label{fig:applications}
    \vspace{-6mm}
\end{figure*}

\vspace{-1mm}
Our contributions are:
\textit{(i)} we introduce \methodName, a new progressive token guidance framework for training-free human motion control, which progressively guides motion from global structure to fine-grained details across temporal scales;
\textit{(ii)} we propose WSI and CSC guidance mechanisms that account for within-scale token interdependencies and the estimated effects of remaining finer scales to improve consistency with the final control target;
and \textit{(iii)} experiments show that \methodName achieves strong motion control quality, while generalizing to unseen differentiable control objectives without control-specific training.

\vspace{-2mm}
\section{Related Works}
\label{Sec:relatedwork}
\vspace{-2mm}

\textbf{Generative models for human motion synthesis.}
Human motion generation seeks to synthesize natural and diverse movements from various conditions, including action labels~\cite{guo2020humanact12,petrovich2021actor,lucas2022posegpt,xu2023actformer,tan2025think}, audio~\cite{li2021early_music2dance,siyao2022bailando,tseng2023edge,zhou2023ude,zhang2025motion_anything}, style~\cite{song2024arbitrary_style,guo2024generative_stylization,zhong2024smoodi,kim2025personabooth}, and text~\cite{ahuja2019early_language2pose,guo2022tm2t,li2025lamp,wang2025you}.
Early deterministic methods~\cite{holden2016motion_cnn,fragkiadaki2015motion_rnn,ahuja2019early_language2pose,mao2019early_pose_history,aksan2021motion_transformer} often regress toward averaged poses and produce overly smooth motions, motivating probabilistic models based on GANs~\cite{wang2020motion_gan,raab2023modi}, conditional VAEs~\cite{zhao2022coins,petrovich2022temos}, and diffusion or flow-based models~\cite{tevet2023mdm,dabral2023mofusion,zhang2024motiondiffuse,chen2023mld,zhou2023ude,zhang2023finemogen,kong2023priority,wang2023fg,zhuo2025infinidreamer,jin2023act_as}.
While diffusion models generate realistic and diverse motions, iterative denoising incurs substantial sampling cost~\cite{zhou2024emdm}.
Accelerated methods~\cite{zhou2024emdm,dai2024motionlcm,jiang2025motionpcm} reduce denoising steps, while autoregressive diffusion models~\cite{han2024amd,shi2024interactive,zhao2025dartcontrol,xiao2025motionstreamer} support streaming generation conditioned on historical poses and evolving text inputs.
Alternatively, token-based models use vector quantization and VQ-VAE~\cite{esser2021taming,van2017vqvae} to represent motion as discrete tokens, generated through autoregressive next-token prediction~\cite{zhang2023t2m_gpt,jiang2023motiongpt,zhang2024motiongpt,zhou2024avatargpt,lu2025scamo} or masked modeling~\cite{guo2024momask,pinyoanuntapong2024mmm,yuan2024mogents,guo2025snapmogen}.
These models improve sampling efficiency, but precise control remains challenging due to non-differentiable sampling and the limited precision of discrete codebooks.
Recent works~\cite{zhang2025kinmo,zheng2026moscale,hwang2026scalemogen} further explore hierarchical motion representations for coarse-to-fine generation and editing.
However, they primarily focus on generation and editing rather than progressively guiding token distributions across temporal scales toward unseen control objectives.
In contrast, \methodName turns the multi-scale token hierarchy into a structured control process, where guided tokens at each scale condition subsequent finer-scale predictions without control-specific training.

\noindent\textbf{Controllable motion synthesis.}
Recent methods synthesize motion under diverse control objectives, including object contacts~\cite{taheri2022goal,xu2023interdiff,diller2024cghoi,kulkarni2024nifty,li2024chois}, scene awareness~\cite{rempe2023trace,huang2023scenediffuser,diomataris2024wandr,yi2024tesmo,wang2024move,jiang2024trumans}, humanoid control~\cite{peng2021amp,luo2023perpetual_humanoid,tevet2025closd,zhang2025primal,wu2025uniphys}, and physical constraints~\cite{xie2021physics_based,yuan2023physdiff,ji2025pomp,liu2024progmogen,li2025morph}.
In text-driven motion, precise joint-trajectory control while preserving motion quality, diversity, and text alignment remains challenging~\cite{xie2024omnicontrol,guo2025motionlab}.
MDM~\cite{tevet2023mdm} and PriorMDM~\cite{shafir2024priormdm} support spatial control through motion inpainting during denoising, while classifier guidance~\cite{dhariwal2021diffusion,bansal2024universal_guidance} provides flexible conditioning.
GMD~\cite{karunratanakul2023gmd} handles sparse root-location control using two-stage guided diffusion, while OmniControl~\cite{xie2024omnicontrol} extends spatial guidance to individual joints.
DNO~\cite{karunratanakul2024dno} optimizes latent diffusion noise for test-time control, while DART~\cite{zhao2025dartcontrol} supports control through autoregressive diffusion.
More recently, MIC~\cite{hui2026mic} formulates diffusion sampling as stochastic control to handle both differentiable objectives and non-differentiable criterion-based constraints, with adaptive constraint coordination.
MotionLCM~\cite{dai2024motionlcm} incorporates explicit control through ControlNet~\cite{zhang2023controlnet}.
ProjFlow~\cite{watanabe2026projflow} achieves exact satisfaction of linear spatial constraints through projection sampling, but does not directly support general nonlinear objectives.
Token-based methods such as TLControl~\cite{wan2024tlcontrol} and CoMo~\cite{huang2024como} manipulate discrete motion codes, while MaskControl~\cite{pinyoanuntapong2025maskcontrol} combines ControlNet training with iterative logit optimization during masked generation~\cite{guo2024momask}.
In contrast, \methodName leverages the multi-scale token hierarchy for progressive training-free control, where guided tokens at each scale condition subsequent finer-scale predictions.
To further improve guidance consistency across the hierarchy, we introduce complementary Within-Scale Interdependency (WSI) and Cross-Scale Consistency (CSC) guidance mechanisms.

\vspace{-2mm}
\section{Human Motion Control with \methodName}
\vspace{-2mm}
\label{Sec:method}

We pursue a training-free control framework that generates high-quality human motion while satisfying user-specified objectives.
Given a text prompt $\bm{y}$, we model a motion sequence
$\bm{x}=\{\bm{x}_1,\ldots,\bm{x}_T\}\in\mathbb{R}^{T\times D}$
through the conditional distribution $p(\bm{x}\mid\bm{y})$,
where $T$ is the motion length and $D$ is the pose dimension.
At inference, a differentiable control objective $G$ specifies additional constraints for the motion to follow, such as target joint locations or obstacle avoidance, without requiring control-specific training.

Our key idea is to align the control process with the coarse-to-fine motion hierarchy.
Sec.~\ref{subsec:motion_var} introduces the multi-scale motion prior, which represents motion using discrete token sequences at progressively finer temporal resolutions.
In Sec.~\ref{subsec:token_guidance}, we progressively guide all current-scale token distributions according to $G$, and the guided tokens become the conditioning prefix for subsequent finer-scale predictions.
This establishes a structured control process that progressively propagates the control objective from global motion structure to fine-grained details, allowing finer scales to build upon the controlled coarse-scale representation while maintaining motion consistency.

\vspace{-2mm}
\subsection{Multi-scale Motion Prior}
\vspace{-1mm}
\label{subsec:motion_var}

To support progressive coarse-to-fine control, we adopt a multi-scale VQ-VAE and an autoregressive transformer that represent and generate motion through discrete token sequences at progressively finer temporal resolutions.

\textbf{Multi-scale Motion Tokenization.}
Following~\cite{zheng2026moscale,hwang2026scalemogen}, the encoder $\mathcal{E}$ maps a motion sequence $\bm{x}$ into latent features $\bm{f}\in\mathbb{R}^{L\times d}$, which are quantized into $K$ token sequences $\bm{z}=\{\bm{z}_k\}_{k=1}^K$ at progressively finer temporal resolutions $(L_1\le\ldots\le L_K=L)$.
Using residual quantization~\cite{zeghidour2021quant_dropout,lee2022residual_quant,guo2024momask}, at each scale $k$, the quantizer temporally downsamples the current residual features $\bm{r}_k$ and assigns each latent vector to its nearest entry in a learnable codebook $\mathcal{Z}\in\mathbb{R}^{V\times d}$:
\begin{equation}
\label{eq:quantization}
\bm{z}_k^{(l)}
=
\arg\min_{v\in\{1,\ldots,V\}}
\left\|
\bm{e}_v-\mathrm{Downsample}_{k}^{(l)}(\bm{r}_k)
\right\|_2,
\end{equation}
where $\bm{e}_v$ is the $v$-th codebook embedding, $l$ denotes the temporal position, and $\bm{r}_1=\bm{f}$.
The quantized embeddings are then upsampled to the finest latent resolution and subtracted from the residual:
\begin{equation}
\label{eq:f_hat_lookup}
\hat{\bm{f}}_k
=
\mathrm{Upsample}_{K}\left(
\mathrm{Lookup}(\mathcal{Z},\bm{z}_k)
\right),\quad
\bm{r}_{k+1}=\bm{r}_k-\hat{\bm{f}}_k.
\end{equation}
The quantized features from all scales are accumulated and decoded to reconstruct the motion:
\begin{equation}
\label{eq:vae_decode}
\hat{\bm{f}}=\sum_{k=1}^K\hat{\bm{f}}_k,\qquad
\hat{\bm{x}}=\mathcal{D}(\hat{\bm{f}}).
\end{equation}

\textbf{Multi-scale Transformer.}
We adopt a multi-scale autoregressive transformer~\cite{tian2024var,zheng2026moscale} that models the conditional likelihood:
\begin{equation}
\label{eq:ar_likelihood}
p_\theta(\bm{z}_1,\ldots,\bm{z}_K\mid\bm{y})
=
\prod_{k=1}^K
p_\theta(\bm{z}_k\mid\bm{z}_{<k},\bm{y}),
\end{equation}
where $\bm{z}_{<k}=\{\bm{z}_1,\ldots,\bm{z}_{k-1}\}$ denotes the coarser-scale prefix.
At each scale, the transformer predicts all token distributions in parallel, conditioned on previous scales and text features extracted by CLIP~\cite{radford2021clip}.
The VQ-VAE and transformer are trained using standard reconstruction, quantization, and token prediction losses~\cite{tian2024var,zheng2026moscale}, without control-specific supervision.
Further architecture and training details are provided in the Supplementary.

\begin{figure*}[t]
    \centering
    \includegraphics[width=\linewidth]{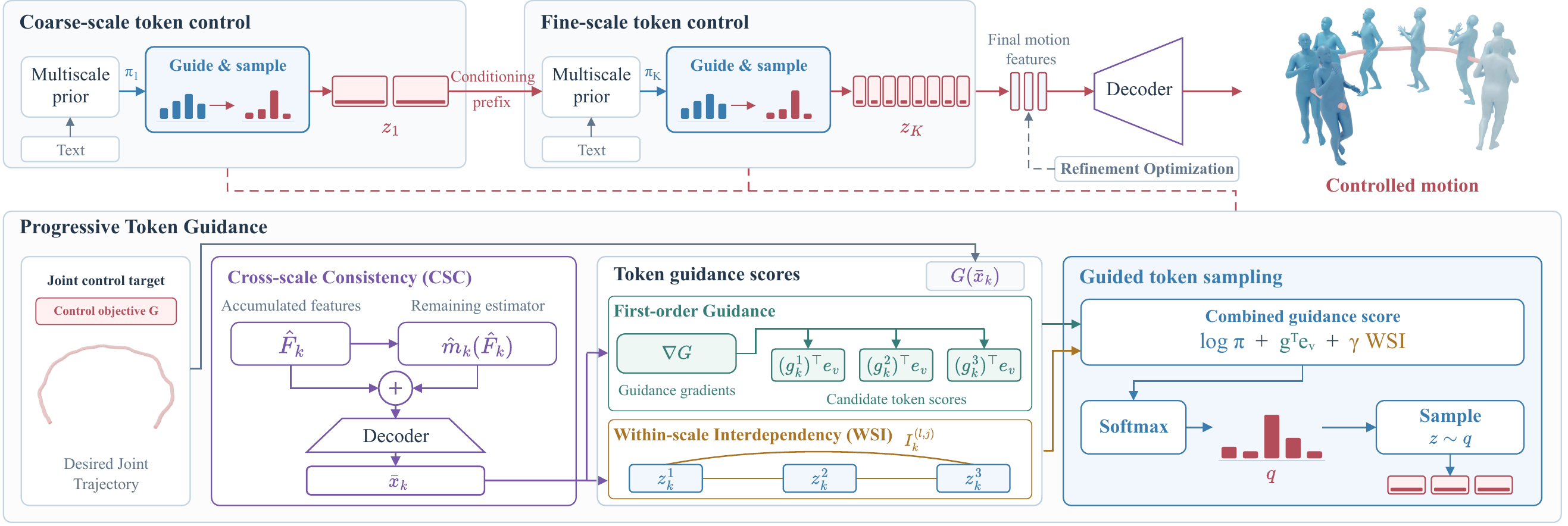}
    \vspace{-5mm}
\caption{
Overview of our progressive coarse-to-fine token guidance framework.
\textbf{Top:} At each scale, the multi-scale motion prior predicts token distributions, which are guided toward the control objective before sampling under our guidance formulation. The controlled tokens condition subsequent finer-scale predictions, followed by a refinement optimization at the last scale.
\textbf{Bottom:} Cross-Scale Consistency (CSC) incorporates estimated information from remaining finer scales into the intermediate motion reconstruction to improve guidance consistency with the final control objective. First-order guidance adjusts token scores using control-objective gradients, while Within-Scale Interdependency (WSI) accounts for the combined effects of token choices. The combined scores determine the guided token distributions.
}

    \label{fig:guidance}
    \vspace{-6mm}
\end{figure*}

\subsection{Progressive Coarse-to-fine Token Guidance}
\vspace{-1mm}
\label{subsec:token_guidance}

At inference, user-specified controls are expressed as a differentiable scalar loss function $G(\hat{\bm{x}})$, which measures how well the generated motion satisfies the target objective (e.g., distance between generated and target joint positions).
Our control process is depicted in Fig.~\ref{fig:guidance}. At each scale, the guided token sequence is sampled and appended to the prefix for the next scale. The controlled coarser-scale tokens thereby establish the global motion structure that conditions subsequent finer-scale predictions, yielding a progressive coarse-to-fine control process.

\textbf{First-order Guidance for Discrete Token.}
At each scale $k$, the multi-scale transformer predicts the prior token distributions $\pi_k^{(l)}(v)=p_\theta(\bm{z}_k^{(l)}=v\mid\bm{z}_{<k},\bm{y})$ for all positions $l$. 
Following Bayes' rule, the goal-conditioned posterior at scale $k$ is given by:
\begin{equation}
\label{eq:bayes_posterior}
\resizebox{0.95\linewidth}{!}{$\displaystyle
p(\bm{z}_k|\bm{z}_{<k},G)
=
\frac{
    p(G|\bm{z}_k,\bm{z}_{<k})
    \pi_k(\bm{z}_k)
}{
    \sum_{\bm{z}'_k\in\{1,\ldots,V\}^{L_k}}
    p(G|\bm{z}'_k,\bm{z}_{<k})
    \pi_k(\bm{z}'_k)
}.
$}
\end{equation}
where the goal likelihood is defined as an energy-based function of the control objective $p(G\mid\hat{\bm{x}})\propto\exp(-G(\hat{\bm{x}}))$.
Here, we omit the text condition $\bm{y}$ for simplicity, as it remains fixed during guidance.
Directly calculating this posterior is intractable since the normalization requires evaluating $V^{L_k}$ possible token combinations, while non-differentiable discrete sampling makes direct gradient-based guidance difficult.

Inspired by the factorized first-order discrete diffusion guidance~\cite{schiff2025simple}, we approximate the goal likelihood through a first-order Taylor expansion in the continuous codebook embedding space:
\begin{equation}
\label{eq:first_order_likelihood}
\resizebox{0.95\linewidth}{!}{$\displaystyle
p(G|\bm{z}_k,\bm{z}_{<k})
\approx
p(G|\bar{\bm{x}}_k)
\exp\left[
\sum_{l=1}^{L_k}
\left(\bm{e}_k^{(l)}-\bar{\bm{e}}_k^{(l)}\right)^\top
\nabla_{\bar{\bm{e}}_k^{(l)}}
\log p(G|\bar{\bm{x}}_k)
\right],
$}
\end{equation}
where $\bm{e}_k^{(l)}$ denotes the codebook embedding corresponding to token $\bm{z}_k^{(l)}$.
We choose the prior expected embedding $\bar{\bm{e}}_k^{(l)}=\sum_{v=1}^{V}\pi_k^{(l)}(v)\bm{e}_v$ as the Taylor expansion point, which represents the model's expectation about the current token before guidance. Since the prior conditions on the controlled tokens from earlier scales, this expansion point also incorporates the control information established at coarser scales.
We decode these expected embeddings together with the fixed coarser-scale prefix to obtain the intermediate motion $\bar{\bm{x}}_k$.
The guidance gradient $\nabla_{\bar{\bm{e}}_k^{(l)}}\log p(G|\bar{\bm{x}}_k)$ can then be computed for all positions at the same time, without decoding each candidate separately.
This yields additive guidance scores across token positions, enabling efficient categorical guidance.
Unlike discrete diffusion guidance~\cite{schiff2025simple} which operates on discrete denoising transitions, our formulation guides the learned motion-codebook embeddings within a multi-scale autoregressive hierarchy, where controlled tokens condition subsequent finer-scale predictions.

Nevertheless, first-order guidance has two limitations. First, its additive scores overlook the combined effects of token choices, potentially leading to suboptimal token combinations. Second, guidance evaluated on intermediate motion reconstruction does not account for remaining scales, so a token that improves the intermediate objective may be inconsistent with the final outcome. We address these limitations through two complementary mechanisms, namely \textit{Within-Scale Interdependency (WSI)} and \textit{Cross-Scale Consistency (CSC)} guidance, introduced below.

\textbf{Within-Scale Interdependency Guidance (WSI).}
The first-order method assigns guidance scores to each token independently, without considering how their combined effects influence the control objective. For example, two candidate tokens may individually move a joint closer to its target, but selecting them together may overshoot the target. To capture such interdependency, we measure how the combined effect of two token choices differs from their individual effects on the control objective.

Specifically, the motion change induced by selecting codebook entry $v$ at position $l$ can be approximated as $\Delta\bm{x}_k^{(l)}(v)\approx\bm{J}_k^{(l)}(\bm{e}_v-\bar{\bm{e}}_k^{(l)})$, where $\bm{J}_k^{(l)}$ is the Jacobian of the reconstructed motion with respect to the token embedding at position $l$, evaluated at the prior expected embeddings. This motion change lets us evaluate how selecting candidate token $v$ affects the control objective. We define the interdependency score between codebook entries $v$ and $w$ at two positions $l$ and $j$ as:
\begin{equation}
\label{eq:token_interdependency}
\resizebox{0.98\columnwidth}{!}{$\displaystyle
\begin{aligned}
I_k^{(l,j)}(v,w) = -\Bigl(
&\underbrace{
\Big[G(\bar{\bm{x}}_k+\Delta\bm{x}_k^{(l)}(v)
+\Delta\bm{x}_k^{(j)}(w))
+G(\bar{\bm{x}}_k)\Big]
}_{\text{Combined effect}}\\
&-\underbrace{
\Big[G(\bar{\bm{x}}_k+\Delta\bm{x}_k^{(l)}(v))
+G(\bar{\bm{x}}_k+\Delta\bm{x}_k^{(j)}(w))\Big]
}_{\text{Individual effects}}
\Bigr).
\end{aligned}
$}
\end{equation}
Intuitively, this score measures how the combined effect of two candidate tokens differs from their individual effects on the control objective. A positive score ($I_k^{(l,j)}>0$) indicates complementary choices that reduce the control objective more than their individual effects suggest, whereas $I_k^{(l,j)}<0$ indicates less favorable combinations due to potentially redundant or conflicting effects.  We then calculate the Within-Scale Interdependency (WSI) Guidance term by aggregating the pairwise scores between a candidate token and other positions at the current scale:
\begin{equation}
\label{eq:wsg}
\mathrm{WSI}_k^{(l)}(v)
=
\frac{1}{V(L_k-1)}
\sum_{j\ne l}^{L_k}
\sum_{w=1}^{V}
I_k^{(l,j)}(v,w).
\end{equation}
This term captures how well each candidate token combines with other token choices, providing a correction to the individual first-order guidance score. In practice, we use  top-$M$ candidates~\cite{hewitt2022topk} ranked by the prior probabilities at each position, while excluding unlikely token choices.

\begin{table*}[!t]
\caption{Joint-controlled motion generation evaluation on HumanML3D dataset~\cite{guo2022humanml3d}. \textbf{Pelvis} reports the results on controlling the human's root trajectory. \textbf{Average} indicates the average performance over all evaluated joints. $^\dagger$ denotes method specifically designed for exact linear spatial joint constraints. $\rightarrow$ indicates closer to the ground-truth value is better. Best in \textbf{bold}, second best \underline{underlined}. }
\vspace{-2ex}
\centering
\resizebox{\linewidth}{!}{%
\begin{tabular}{l@{\hspace{8pt}} l c c c c c c r r r r}
\toprule
\multirow{2}{*}{\textbf{Joint}} & \multirow{2}{*}{\textbf{Method}}
& \multirow{2}{*}{\makecell{\textbf{R-Precision}\\ \textbf{(Top-1)}$\uparrow$}}
& \multirow{2}{*}{\makecell{\textbf{R-Precision}\\ \textbf{(Top-2)}$\uparrow$}}
& \multirow{2}{*}{\makecell{\textbf{R-Precision}\\ \textbf{(Top-3)}$\uparrow$}}
& \multirow{2}{*}{\textbf{FID}$\downarrow$}
& \multirow{2}{*}{\textbf{Diversity}$\rightarrow$}
& \multirow{2}{*}{\makecell{\textbf{Skating}\\\textbf{Ratio}$\downarrow$}}
& \multirow{2}{*}{\makecell{\textbf{Trajectory}\\ \textbf{Error}$\downarrow$}}
& \multirow{2}{*}{\makecell{\textbf{Location}\\ \textbf{Error}$\downarrow$}}
& \multirow{2}{*}{\makecell{\textbf{Average}\\ \textbf{Error} (cm)$\downarrow$}}
& \multirow{2}{*}{\makecell{\textbf{Run}\\\textbf{time} (s)$\downarrow$}}\\
~ & ~ & ~ & ~ & ~ & ~ & ~ & ~ & ~ & ~ & ~ & ~ \\
\midrule
& {GT} & 0.511 & 0.703 & 0.797 & 0.002 & 9.503 & & 0.00 & 0.00 & 0.00 & -- \\
\midrule
\sectionrow{Training-free (Without ControlNet)}
\midrule
\multirow{8}{*}{\textbf{Pelvis}}
& MDM~\cite{tevet2023mdm}
& 0.331 & 0.515 & 0.623 & 0.669 & 9.225 & 0.1223 & 45.05 & 32.77 & 63.82 & 39.66 \\
& PriorMDM~\cite{shafir2024priormdm}
& 0.338 & 0.504 & 0.572 & 0.444 & 9.138 & 0.0911 & 32.48 & 22.16 & 43.83 & 46.86 \\
& GMD~\cite{karunratanakul2023gmd}
& 0.367 & 0.550 & 0.655 & 0.527 & 9.234 & 0.0899 & 9.02 & 3.86 & 14.02 & 108.41 \\
& OmniControl~\cite{xie2024omnicontrol}
& 0.419 & 0.615 & 0.688 & 0.367 & 9.441 & 0.0681 & 8.85 & 3.71 & 11.14 & 119.54 \\
& TLControl~\cite{wan2024tlcontrol}
& 0.488 & 0.671 & 0.762 & 0.302 & 9.548 & 0.0702 & \underline{1.03} & 0.53 & 2.43 & 34.58 \\
& MaskControl~\cite{pinyoanuntapong2025maskcontrol}
& \underline{0.516} & \underline{0.707} & \underline{0.803} & 0.236 & \underline{9.475} & 0.0679 & 1.16 & \underline{0.21} & 2.08 & 39.02 \\
& ProjFlow$^{\dagger}$~\cite{watanabe2026projflow}
& 0.497 & 0.691 & 0.785 & \underline{0.087} & 9.663 & \underline{0.0650} & \textbf{0.00} & \textbf{0.00} & \textbf{0.00} & \underline{9.46} \\
\cmidrule(lr){2-12}
& \textbf{\methodName (ours)}
& \textbf{0.533} & \textbf{0.725} & \textbf{0.825} & \textbf{0.067} & \textbf{9.512} & \textbf{0.0477} & \textbf{0.00} & \textbf{0.00} & \underline{0.25} & \textbf{4.59} \\
\midrule
\multirow{4}{*}{\textbf{Average}}
& OmniControl~\cite{xie2024omnicontrol}
& 0.416 & 0.612 & 0.683 & 0.359 & \underline{9.482} & 0.0692 & 14.66 & 7.14 & 12.15 & 120.02 \\
& MaskControl~\cite{pinyoanuntapong2025maskcontrol}
& \underline{0.520} & \underline{0.712} & \underline{0.809} & 0.298 & 9.414 & 0.0668 & \underline{0.97} & \underline{0.15} & 1.38 & 39.05 \\
& ProjFlow$^{\dagger}$~\cite{watanabe2026projflow}
& 0.501 & 0.688 & 0.785 & \underline{0.099} & 9.744 & \underline{0.0619} & \textbf{0.00} & \textbf{0.00} & \textbf{0.00} & \underline{9.51} \\
\cmidrule(lr){2-12}
& \textbf{\methodName (ours)}
& \textbf{0.535} & \textbf{0.725} & \textbf{0.823} & \textbf{0.092} & \textbf{9.491} & \textbf{0.0467} & \textbf{0.00} & \textbf{0.00} & \underline{0.31} & \textbf{4.64} \\
\midrule
\sectionrow{With ControlNet}
\midrule
\multirow{4}{*}{\textbf{Average}}
& MotionLCM~\cite{dai2024motionlcm}
& 0.488 & 0.671 & 0.768 & 0.542 & 9.275 & 0.0598 & 18.89 & 8.14 & 19.58 & -- \\
& OmniControl~\cite{xie2024omnicontrol}
& 0.423 & 0.619 & 0.693 & 0.204 & \underline{9.520} & \underline{0.0549} & \underline{3.68} & \underline{1.01} & 3.29 & 134.06 \\
& MaskControl~\cite{pinyoanuntapong2025maskcontrol}
& \underline{0.516} & \underline{0.708} & \underline{0.802} & \underline{0.075} & 9.375 & 0.0556 & \textbf{0.00} & \textbf{0.00} & \underline{1.11} & \underline{39.21} \\
\cmidrule(lr){2-12}
& \textbf{\methodName (ours)}
& \textbf{0.541} & \textbf{0.742} & \textbf{0.828} & \textbf{0.035} & \textbf{9.498} & \textbf{0.0406} & \textbf{0.00} & \textbf{0.00} & \textbf{0.00} & \textbf{4.91} \\
\bottomrule
\end{tabular}%
}
\label{tab:main_eval}
\vspace{-3ex}
\end{table*}

\textbf{Cross-Scale Consistency Guidance (CSC).}
Although WSI accounts for interdependencies between token choices at the current scale, guidance is still evaluated on an intermediate motion reconstruction with missing information from the remaining scales that have not yet been generated. Since these scales may further alter the motion, a token that improves the intermediate control objective may not improve the final control objective, leading to inconsistent guidance across scales.  To reduce this mismatch, we estimate the information from remaining finer scales without explicitly generating their tokens.

Let $\hat{\bm F}_k=\sum_{s=1}^{k}\hat{\bm f}_s$ denote the accumulated quantized representation at scale $k$, with $\hat{\bm F}_K=\hat{\bm f}$ as in Eq~\ref{eq:vae_decode}. We estimate the remaining-scale information through the following conditional expectation:
\begin{equation}
\label{eq:cross_scale_expectation}
m_k(\hat{\bm F}_k)
=
\mathbb E[\hat{\bm F}_K-\hat{\bm F}_k\mid\hat{\bm F}_k].
\end{equation}
This conditional expectation is the optimal solution for minimizing the expected squared prediction error. We approximate this using a scale-wise linear estimator $\hat m_k(\hat{\bm F}_k)=\bm W_k\hat{\bm F}_k+\bm b_k$, fitted through linear least squares on motion tokens sampled from the frozen multi-scale prior. The estimator is fixed during inference, requiring no control-specific training or additional autoregressive predictions. We then incorporate the estimated remaining information into the motion reconstruction used for guidance at current scale $k$:
\begin{equation}
\label{eq:cross_scale_guidance}
\resizebox{0.95\linewidth}{!}{$\displaystyle
\bar{\bm x}_k
=
\mathcal{D}\left(
\hat{\bm F}_{k-1}
+\hat{\bm f}_k(\bar{\bm e}_k)
+\hat m_k\left(\hat{\bm F}_{k-1}+\hat{\bm f}_k(\bar{\bm e}_k)\right)
\right).
$}
\end{equation}
where $\hat{\bm f}_k(\bar{\bm e}_k)$ is obtained from the prior expected embeddings $\bar{\bm e}_k$ using the temporal upsampling described in Sec.~\ref{subsec:motion_var}. By incorporating the estimated information across remaining scales, CSC provides an approximate motion reconstruction for evaluating current-scale guidance under the final control objective. This encourages consistent guidance at each scale toward the final control outcome.

\textbf{Progressive Token Guidance.}
Combining the first-order guidance in Eq.~\ref{eq:first_order_likelihood}, the WSI correction in Eq.~\ref{eq:wsg}, and the CSC reconstruction in Eq.~\ref{eq:cross_scale_guidance}, we obtain the final guided token distribution:
\begin{equation}
\label{eq:progressive_guidance}
\resizebox{0.95\columnwidth}{!}{$\displaystyle
q_k^{(l)}(v)
=
\operatorname{softmax}_{v}
\left[
\log\pi_k^{(l)}(v)
+
(\bm{g}_k^{(l)})^\top\bm{e}_v
+
\gamma\,\mathrm{WSI}_k^{(l)}(v)
\right].
$}
\end{equation}
where $\bm{g}_k^{(l)}=\nabla_{\bar{\bm{e}}_k^{(l)}}\log p(G\mid\bar{\bm{x}}_k)$ is the first-order guidance gradient and $\gamma$ controls the strength of the WSI correction.
The first-order term guides individual token choices toward the control objective, while WSI adjusts their scores to account for the combined effects of other token choices.
Both terms are evaluated using the CSC reconstruction $\bar{\bm{x}}_k$, incorporating the estimated effects of remaining finer scales into the guidance.

At each scale, we update all token distributions in parallel and sample the guided tokens $\bm{z}_k^{(l)}\sim q_k^{(l)}$.
These tokens then condition subsequent finer-scale predictions, progressively guiding the control objective from global motion structure to fine-grained details.
We provide more details of our progressive guidance process in the Supplementary.

\textbf{Refinement Optimization.}
Although our token guidance can effectively steer the discrete token distribution toward the control objective, the finite codebook may not represent the fine-grained motion details needed to precisely satisfy the target. This is an inherent limitation of discrete token-based models~\cite{wan2024tlcontrol}. Following~\cite{wan2024tlcontrol,pinyoanuntapong2025maskcontrol,watanabe2026projflow}, we further apply continuous refinement optimization after the final scale, at inference. Given the final latent motion features $\hat{\bm f}$ generated by our progressive guidance, we seek the smallest correction that improves control satisfaction while preserving the generated motion:
\begin{equation}
\label{eq:refinement_optimization}
\hat{\bm f}^{*}
=
\arg\min_{\bm f}
\left[
G(\mathcal{D}(\bm f))
+
\frac{1}{2}\|\bm f-\hat{\bm f}\|_2^2
\right],
\end{equation}
The optimization is initialized from the generated $\hat{\bm f}$ and updated through the frozen decoder $\mathcal{D}$ using gradient descent, without additional training.

\section{Experiments}
\label{Sec:exp}
We evaluate our method and other baselines on several settings using the popular HumanML3D~\cite{guo2022humanml3d} text-to-motion dataset. The dataset contains 14,616 diverse action sequences paired with a total of 44,970 text descriptions, where the motions are aggregated from AMASS~\cite{mahmood2019amass} and HumanAct12~\cite{guo2020humanact12}. More implementation details are provided in the Supplementary.

\begin{table*}[ht]
\caption{Evaluation on different unseen human-scene interaction control tasks. The control objectives and metrics are adopted from ProgMoGen~\cite{liu2024progmogen}. Best in \textbf{bold}, second best \underline{underlined}.}
\vspace{-2ex}
\centering
\scriptsize
\setlength{\tabcolsep}{4pt}
\renewcommand{\arraystretch}{1.1}
\resizebox{\linewidth}{!}{%
\begin{tabular}{l c c c c c | c c c | c c c}
\toprule
\multirow{2}{*}{\textbf{Method}} &
\multicolumn{5}{c|}{\textbf{Task HSI-1: Head height constraint}} &
\multicolumn{3}{c|}{\textbf{Task HSI-2: Avoiding overhead barrier}} &
\multicolumn{3}{c}{\textbf{Task HSI-3: Walking inside a square}} \\
\cmidrule(lr){2-6} \cmidrule(lr){7-9} \cmidrule(lr){10-12}
& \makecell{\textbf{R-Precision}\\ \textbf{(Top-3)}$\uparrow$}
& \textbf{FID}$\downarrow$
& \makecell{\textbf{Skating}\\\textbf{Ratio}$\downarrow$}
& \makecell{\textbf{Constraint}\\ \textbf{Error}$\downarrow$}
& \makecell{\textbf{Unsuccess}\\ \textbf{Rate}$\downarrow$}
& \makecell{\textbf{Skating}\\\textbf{Ratio}$\downarrow$}
& \makecell{\textbf{Max}\\\textbf{Acceleration}$\downarrow$}
& \makecell{\textbf{Constraint}\\\textbf{Error}$\downarrow$}
& \makecell{\textbf{Skating}\\\textbf{Ratio}$\downarrow$}
& \makecell{\textbf{Max}\\\textbf{Acceleration}$\downarrow$}
& \makecell{\textbf{Constraint}\\\textbf{Error}$\downarrow$} \\
\midrule
ProgMoGen~\cite{liu2024progmogen}
& 0.601 & 0.596 & 0.074 & 0.011 & 0.092
& 0.220 & 0.153 & 0.073
& 0.121 & 0.093 & 0.014 \\
DNO~\cite{karunratanakul2024dno}
& 0.578 & 0.531 & 0.093 & 0.017 & 0.131
& 0.209 & 0.174 & 0.054
& 0.158 & 0.101 & 0.015 \\
MIC~\cite{hui2026mic}
& 0.635 & 0.494 & 0.074 & \underline{0.009} & 0.068
& \underline{0.172} & 0.140 & 0.009
& 0.112 & 0.089 & \underline{0.004} \\
MaskControl~\cite{pinyoanuntapong2025maskcontrol}
& \underline{0.697} & \underline{0.286} & \underline{0.068} & \textbf{0.000} & \underline{0.002}
& 0.173 & \underline{0.114} & \underline{0.002}
& \underline{0.099} & \underline{0.082} & \textbf{0.000} \\
\textbf{\methodName (ours)}
& \textbf{0.742} & \textbf{0.188} & \textbf{0.056} & \textbf{0.000} & \textbf{0.000}
& \textbf{0.141} & \textbf{0.099} & \textbf{0.000}
& \textbf{0.079} & \textbf{0.068} & \textbf{0.000} \\
\bottomrule
\end{tabular}%
}
\label{tab:hsi_tasks_merged}
\vspace{-2ex}
\end{table*}

\begin{table*}[ht]
\caption{Evaluation on unseen geometric and human-object interaction control tasks. The control objectives and metrics follow ProgMoGen~\cite{liu2024progmogen}. Best in \textbf{bold}, second best \underline{underlined}.}
\vspace{-2ex}
\centering
\scriptsize
\setlength{\tabcolsep}{4pt}
\renewcommand{\arraystretch}{1.0}
\begin{tabular*}{\linewidth}{@{\extracolsep{\fill}}l c c c | c c c}
\toprule
\multirow{2}{*}{\textbf{Method}} &
\multicolumn{3}{c|}{\textbf{Task GEO-1: Hand touching a wall}} &
\multicolumn{3}{c}{\textbf{Task HOI-1: Moving an object}} \\
\cmidrule(lr){2-4} \cmidrule(lr){5-7}
& \textbf{Skating Ratio}$\downarrow$
& \textbf{Max Acceleration}$\downarrow$
& \textbf{Constraint Error}$\downarrow$
& \textbf{Skating Ratio}$\downarrow$
& \textbf{Max Acceleration}$\downarrow$
& \textbf{Constraint Error}$\downarrow$ \\
\midrule
ProgMoGen~\cite{liu2024progmogen}
& 0.109 & 0.0969 & 0.025
& 0.104 & 0.0735 & 0.031 \\
DNO~\cite{karunratanakul2024dno}
& 0.121 & 0.1090 & 0.038
& 0.114 & 0.0900 & 0.032 \\
MIC~\cite{hui2026mic}
& 0.102 & \underline{0.0880} & \underline{0.008}
& \textbf{0.036} & \textbf{0.0620} & \underline{0.004} \\
MaskControl~\cite{pinyoanuntapong2025maskcontrol}
& \underline{0.094} & 0.1250 & 0.011
& 0.118 & 0.0907 & 0.018 \\
\textbf{\methodName (ours)}
& \textbf{0.069} & \textbf{0.0871} & \textbf{0.001}
& \underline{0.076} & \underline{0.0637} & \textbf{0.000} \\
\bottomrule
\end{tabular*}
\label{tab:geo_hoi_tasks}
\vspace{-3ex}
\end{table*}

\subsection{Motion Control under Joint Constraints}
\label{subsec:joint_control}

We first evaluate motion control under spatial joint constraints, where the generated motion must follow target 3D joint locations specified at arbitrary frames. In this setting, the control objective is defined as the distance between the specified targets and the corresponding generated joints.
Following~\cite{xie2024omnicontrol,pinyoanuntapong2025maskcontrol} on HumanML3D~\cite{guo2022humanml3d}, we sample control timestamps at five different densities and evaluate motion quality using R-Precision, FID, Diversity, and Foot Skating Ratio, together with control accuracy using Trajectory Error, Location Error, and Average Error.
For fair comparison, all baselines are evaluated using their official implementations under the same settings, hardware, and random seed.

The upper part of Table~\ref{tab:main_eval} evaluates the training-free control setting without ControlNet, where \methodName  guides the pretrained motion prior toward user-defined constraints at inference without additional control-specific training.
In terms of pelvis trajectory control, compared with MaskControl~\cite{pinyoanuntapong2025maskcontrol}, \methodName reduces FID from $0.236$ to $0.067$ and average error from $2.08$ to $0.25$ cm. Although ProjFlow~\cite{watanabe2026projflow} achieves exact constraint satisfaction, it is specifically designed for linear spatial constraints and does not directly support nonlinear control objectives. 
\methodName achieves better text-motion alignment (R-Precision $0.825$ vs. $0.785$), motion quality (FID $0.067$ vs. $0.087$), and foot skating ($0.0477$ vs. $0.0650$), while maintaining a low average control error of $0.25$ cm.
Across all evaluated joints, \methodName maintains strong performance  with FID of $0.092$ and average control error of $0.31$ cm.
These results demonstrate the effectiveness of our progressive token guidance framework in controlling different body joints while preserving motion quality, without control-specific training.

Fig.~\ref{fig:qualitative} further illustrates qualitative examples in controlling multiple joints. OmniControl~\cite{xie2024omnicontrol} exhibits noticeable local error from the target joint trajectories, while MaskControl~\cite{pinyoanuntapong2025maskcontrol} exhibits global inconsistency and foot skating.
In contrast, \methodName produces more coherent motion while closely following the joint targets. These observations are consistent with our quantitative results and suggest the benefit of progressively guiding global motion structure before refining local details.

For completeness, we additionally report results under the standard ControlNet-trained setting following the original protocol of MaskControl~\cite{pinyoanuntapong2025maskcontrol} in the lower part of Table~\ref{tab:main_eval}, where ControlNet~\cite{zhang2023controlnet} is used during training. The results are averaged over pelvis, feet, head, and wrist control.
Under this setting, \methodName also achieves strong motion quality and control accuracy.

\begin{figure}[!t]
    \centering
    \vspace{-3mm}
    \includegraphics[width=\linewidth]{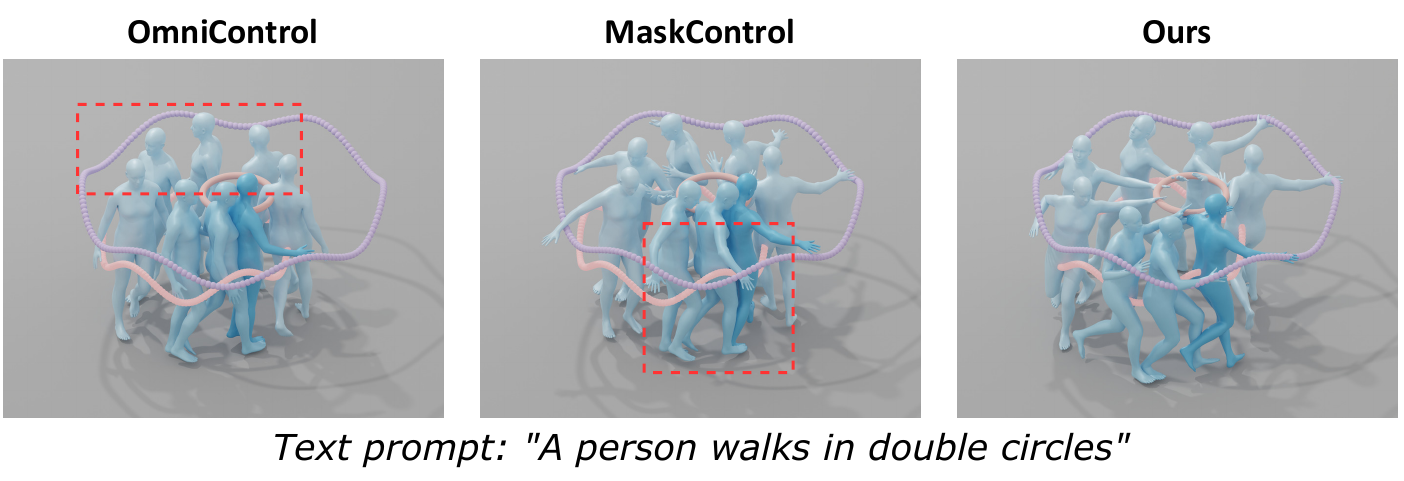}
    \vspace{-7mm}
\caption{
Qualitative comparison of multi-joint motion control.
OmniControl exhibits less accurate local joint control, while MaskControl shows inconsistent foot skating. \methodName better follows the joint targets while maintaining coherent motion.
}
    \label{fig:qualitative}
    \vspace{-6.5mm}
\end{figure}

\begin{figure*}[!t]
    \centering
    \includegraphics[width=\linewidth]{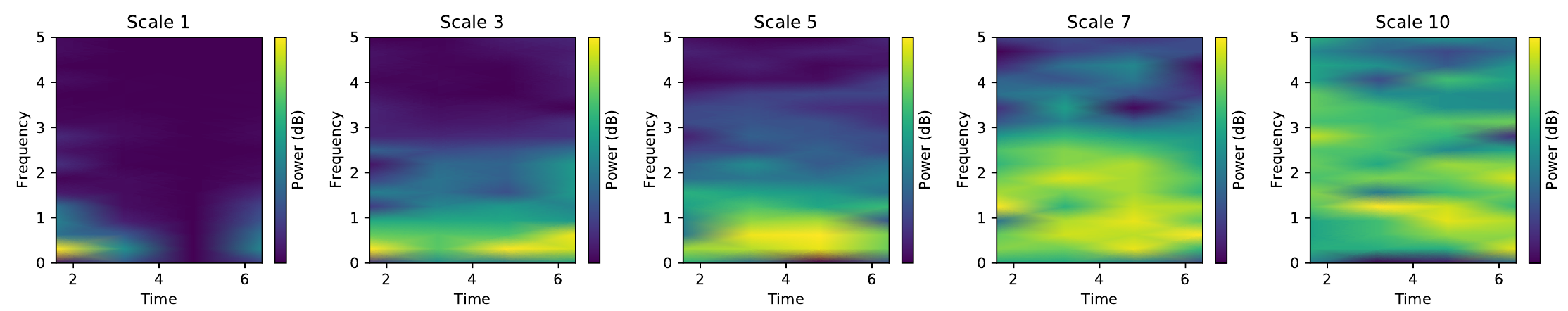}
    \vspace{-7.5mm}
    \caption{Time-frequency spectrogram of the intermediate motions across 5 progressive scales. Darker region indicates lower energy of the  motion.}
    \label{fig:spectrogram}
    \vspace{-5mm}
\end{figure*}

\subsection{Motion Control under Open-Set Constraints}

We further evaluate the flexibility of \methodName on open-set control tasks from ProgMoGen~\cite{liu2024progmogen}, covering human-scene interaction (HSI), geometric constraints (GEO), and human-object interaction (HOI).
The three HSI tasks involve head height control (HSI-1), overhead barrier avoidance (HSI-2), and walking within a specified region (HSI-3). GEO-1 controls hand touching with a wall, while HOI-1 evaluates object manipulation. All tasks are formulated as differentiable control objectives introduced only at inference. Following ProgMoGen~\cite{liu2024progmogen}, we evaluate control accuracy using constraint error and unsuccess rate, and motion quality using foot skating ratio and maximum joint acceleration. For HSI-1, we also report R-Precision and FID.

As shown in Table~\ref{tab:hsi_tasks_merged}, \methodName achieves strong performance across all three HSI tasks.
For HSI-1, compared with MaskControl~\cite{pinyoanuntapong2025maskcontrol}, our method improves R-Precision from $0.697$ to $0.742$ and reduces FID from $0.286$ to $0.188$, while maintaining low control error.
Similarly, for HSI-2 and HSI-3, \methodName achieves the lowest skating ratios and maximum accelerations, with constraint errors approaching zero.
These results demonstrate that our progressive token guidance can adapt to different spatial constraints while preserving motion quality.

Table~\ref{tab:geo_hoi_tasks} further evaluates geometric and human-object interaction control.
For GEO-1, \methodName achieves the lowest constraint error of $0.001$, together with improved skating ratio and maximum acceleration.
For HOI-1, our method reduces constraint error to $0.000$, while maintaining competitive motion smoothness compared with MIC~\cite{hui2026mic}.
Overall, these results demonstrate the flexibility of our progressive control framework in addressing diverse task-specific constraints using the same pretrained motion model.
Additional applications are illustrated in Fig.~\ref{fig:applications}.

\subsection{Ablation Analysis}

\begin{table}[t]
\centering
\caption{
Ablation study on joint trajectory control, evaluating
guidance at different temporal scales and the effects of
Within-Scale Interdependency (WSI) and Cross-Scale Consistency (CSC).
All variants use the same pretrained motion prior.
}
\label{tab:ablation}
\vspace{-2mm}
\footnotesize
\setlength{\tabcolsep}{4pt}
\renewcommand{\arraystretch}{1.12}
\resizebox{\columnwidth}{!}{%
\begin{tabular}{lcccccc}
\toprule
\multirow{2}{*}{\textbf{Method}}
& \multirow{2}{*}{\makecell{\textbf{R-Precision}\\
\textbf{(Top-3)}$\uparrow$}}
& \multirow{2}{*}{\textbf{FID}$\downarrow$}
& \multirow{2}{*}{\makecell{\textbf{Skating}\\
\textbf{Ratio}$\downarrow$}}
& \multirow{2}{*}{\makecell{\textbf{Trajectory}\\
\textbf{Error}$\downarrow$}}
& \multirow{2}{*}{\makecell{\textbf{Location}\\
\textbf{Error}$\downarrow$}}
& \multirow{2}{*}{\makecell{\textbf{Average}\\
\textbf{Error}$\downarrow$}} \\
~ & ~ & ~ & ~ & ~ & ~ & ~ \\
\midrule

Coarse-scale guidance (scales 1--5)
& 0.808 & 0.167 & 0.0619 & 2.11 & 0.79 & 3.63 \\

Fine-scale guidance (scales 6--10)
& 0.806 & 0.206 & 0.0662 & 1.54 & 0.49 & 3.08 \\

Last-scale guidance (scale 10)
& 0.801 & 0.294 & 0.0726 & 3.19 & 1.21 & 4.69 \\

Progressive first-order (all scales)
& 0.814 & 0.128 & 0.0581 & 0.89 & 0.26 & 2.01 \\

\midrule

Progressive guidance + WSI
& 0.821 & 0.088 & 0.0511 & 0.13 & 0.04 & 0.96 \\

Progressive guidance + CSC
& 0.817 & 0.096 & 0.0527 & 0.21 & 0.09 & 1.24 \\

\textbf{Full \methodName{} (WSI + CSC)}
& \textbf{0.825}
& \textbf{0.067}
& \textbf{0.0477}
& \textbf{0.00}
& \textbf{0.00}
& \textbf{0.25} \\

\bottomrule
\end{tabular}%
}
\vspace{-6mm}
\end{table}

\noindent\textbf{Effect of progressive guidance across scales.}
The upper part of Table~\ref{tab:ablation} investigates the effect of applying guidance at different temporal scales using the joint control setting as in Sec.~\ref{subsec:joint_control}.
\textit{Last-scale guidance} yields the highest average error (4.69 cm) and FID (0.294), indicating that controlling only the final scale is insufficient.
Coarse-scale guidance achieves better motion quality than fine-scale guidance (FID $0.167$ vs. $0.206$), while fine-scale guidance provides more accurate control ($3.08$ vs. $3.63$ cm). This suggests that different scales play complementary roles in controlling global motion structure and refining local details.
In contrast, progressive guidance across all scales improves both FID to $0.128$ and average error to $2.01$ cm.
These results demonstrate the effectiveness of our progressive guidance, which combines the global motion structure established at coarse scales with precise local refinements at finer scales.

\noindent\textbf{Effect of WSI and CSC guidance.}
The lower part of Table~\ref{tab:ablation} examines the roles of WSI and CSC in improving the progressive token guidance. By accounting for interactions among token choices, WSI guidance reduces FID from $0.128$ to $0.088$ and average error from $2.01$ to $0.96$ cm.
Similarly, CSC guidance improves FID to $0.096$ and average error to $1.24$ cm, suggesting that estimating the effects of remaining scales helps align intermediate guidance signal with the final control objective.
Ultimately, combining both mechanisms achieves the best results, with FID of $0.067$, average error of $0.25$ cm, and skating ratio of $0.0477$.
These results demonstrate the complementary benefits of WSI and CSC in improving guidance within each scale and maintaining consistency across the motion hierarchy.

\noindent\textbf{Visualization of the progressive coarse-to-fine structure.}
Fig.~\ref{fig:spectrogram} visualizes the motion-energy spectrogram across temporal scales. It can be observed that coarse scales concentrate motion energy at low frequencies, capturing global motion structure such as the overall trajectory and gait, while finer scales progressively introduce higher-frequency dynamics, capturing periodic movements and fine-grained details. 
This observation supports our motivation and illustrates how the representation hierarchy can effectively provide a structured control process from global motion structure to local motion details.

\vspace{-2mm}
\section{Conclusion}
\vspace{-1mm}
\label{sec:Discussion}

We presented \methodName, a progressive token guidance framework for training-free human motion control that turns the multi-scale hierarchy of discrete motion tokens into a structured control process.
By guiding token distributions at successive scales, the controlled coarse-scale tokens condition finer-scale generation, progressively refining motion toward the control objective.
To improve guidance effectiveness, we incorporate Within-Scale Interdependency (WSI) guidance to account for interactions among token choices and Cross-Scale Consistency (CSC) guidance to estimate the effects of remaining finer scales.
Experiments demonstrate strong motion quality and control accuracy across diverse joint, geometric, and interaction constraints without control-specific training. Empirical results further demonstrate the effectiveness of progressive guidance and the complementary benefits of WSI and CSC in improving control across temporal scales.  Our code and model will be released to encourage future research.

\section*{Acknowledgements}
This research was supported by the Australian Research Council Discovery Project DP240101926 and the NVIDIA Academic Grant Program. Nhat Le is a recipient of the Australian Government Research Training Program (RTP) Scholarship, supported by the Commonwealth.

\section*{LLM Usage Statement}
OpenAI ChatGPT was used to assist with literature review, baseline reproduction, and language editing during manuscript preparation. All scientific content, mathematical derivations, and references are the authors' original work, for which the authors take full responsibility.

{\small
\bibliographystyle{ieeenat_fullname}
\bibliography{bib}
}
\end{document}